\documentclass[conference,letterpaper]{IEEEtran}
\IEEEoverridecommandlockouts
\usepackage[utf8]{inputenc} 
\usepackage[T1]{fontenc}
\usepackage{url}
\usepackage{ifthen}
\usepackage{cite}
\usepackage{comment}
\usepackage[cmex10]{amsmath}
\usepackage{mathtools}
\usepackage{pgfplots}
\pgfplotsset{compat=newest}

\usepackage{amsthm}
\usepackage{graphicx}
\usepackage{algorithm, algorithmic}
\usepackage{subcaption}
\usepackage{amsfonts}
\usepackage{thmtools}
\usepackage{booktabs}
\usepackage{float}
\usepackage{amssymb}
\usepackage{amsmath}
\usepackage{amsthm}
\usepackage{xspace}
\usepackage[capitalize]{cleveref}
\usepackage[inline]{enumitem}
\usepackage{cite}
\usepackage{siunitx}
\usepackage{dsfont}
\usepackage{balance}

\usepackage{tikz}
\usetikzlibrary{arrows.meta, positioning, shapes.geometric,bending}
\newtheorem{remark}{Remark}

\declaretheorem[name=Definition]{definition}
\newtheorem{thrm}{Theorem}
\newtheorem{lem}{Lemma}

\DeclarePairedDelimiter\norm{\lVert}{\rVert}

\DeclarePairedDelimiter\bracks{\lbrack}{\rbrack}

\newcommand{\ip}[2]{\left\langle #1, #2 \right\rangle}

\newcommand{\E}[1]{\mathbb{E}\bracks*{#1}}

\newcommand{\algname}{MHLJ}

\newcommand{\mh}{Metropolis-Hastings}

\DeclareMathOperator{\diag}{\mathrm{diag}}

\usepackage{xcolor}
\usepackage[normalem]{ulem}

\begin{document}
\title{The Entrapment Problem in Random Walk Decentralized Learning } 


\author{%
  
  \IEEEauthorblockN{Zonghong Liu and Salim El Rouayheb}
  \IEEEauthorblockA{ 
  Department of Electrical and Computer Engineering\\ Rutgers University, New Brunswick, NJ, USA\\
                    Email: \{zonghong.liu, salim.elrouayheb\}@rutgers.edu}

    \and
    \IEEEauthorblockN{Matthew Dwyer}
    \IEEEauthorblockA{ 
    Network, Cyber, and Computational Sciences Division\\
                    DEVCOM Army Research Laboratory, Adelphi, MD, USA\\
                Email: matthew.r.dwyer7.civ@army.mil}     

 \thanks{This work was supported in part by the Army Research Lab (ARL) under Grant W911NF-21-2-0272 and the National Science Foundation (NSF) under Grant CNS-2148182.}}

\maketitle
\vspace{-0.3cm}
\begin{abstract}
This paper explores decentralized learning in a graph-based setting, where data is distributed across nodes. We investigate a decentralized SGD algorithm that utilizes a random walk to update a global model based on local data. Our focus is on designing the transition probability matrix to speed up convergence. While importance sampling can enhance centralized learning, its decentralized counterpart, using the Metropolis-Hastings (MH) algorithm, can lead to the entrapment problem, where the random walk becomes stuck at certain nodes, slowing convergence. To address this, we propose the  \mh\ with L\'{e}vy Jumps (MHLJ) algorithm, which incorporates random perturbations (jumps) to overcome entrapment. We theoretically establish the convergence rate and error gap of MHLJ and validate our findings through numerical experiments.
   
\end{abstract}

\section{Introduction}

Traditional machine learning typically stores and trains models on a single server. This framework struggles with large-scale data and poses privacy leakage issues. These challenges have led to a shift towards researching distributed learning \cite{zinkevich2010parallelized,richtarik2016parallel}. A particularly focused framework is centralized distributed learning, which requires a central server, suffers from a communication bottleneck \cite{praneeth2019scaffold}, and is vulnerable if the central server fails \cite{gupta2021localnewton, guerraoui2018hidden}.  
Decentralized learning models remove the dependency on a central server.
In this paper, we study the setting of decentralized learning via random walks (RWs), as shown in \cref{Learning}. The data needed to train the global model is held by local devices (nodes) in a network. Moreover, no central server is needed to aggregate the local updates performed in each iteration at the nodes \cite{kairouz2021advances}, alleviating the problems of communication bottleneck, privacy, and failures that come with a centralized setting.
The learning task is accomplished by leveraging the local communication links among the devices.   The learning task can be expressed as:
\begin{align}
    \min_{x\in \mathbb{R}^d}\frac{1}{|V|}\sum_{v\in V}f_v(x),\label{obj}
\end{align}
where $f_v$ is the local loss function of $v$, which depends on the local data.
Existing decentralized learning approaches can be categorized into two main categories: gossip algorithms \cite{boyd2005gossip,boyd2006randomized}, which have been extensively studied, and random walk algorithms \cite{johansson2010randomized,ayache2021private}, which have been garnering increasing interest recently.  
\begin{figure}[t]
\centering
\includegraphics[width=0.8\linewidth]{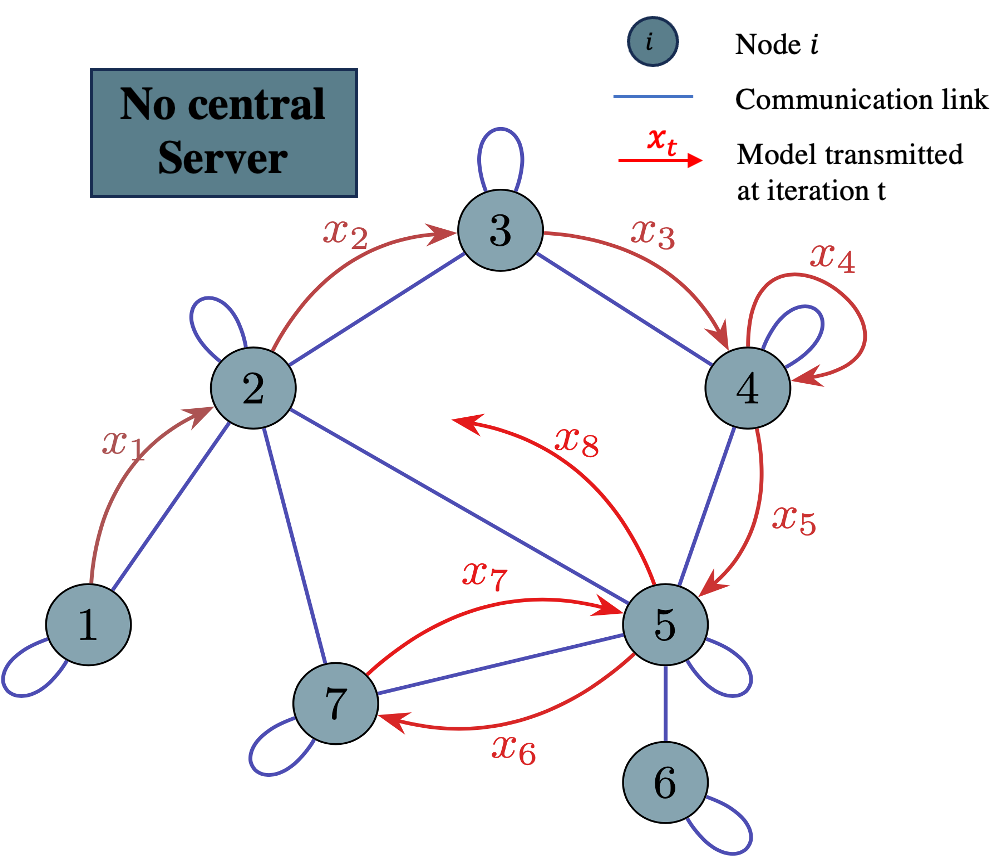}
\caption{\small Decentralized learning via random walk. The model $x$ is carried by a random walk, which is represented by the red arrows. The model is updated using   local data of the visited node in each iteration.}
  \label{Learning}
  \vspace{-0.5cm} 
\end{figure}
In this work, we focus on random walk algorithms due to their overall low communication overhead. We study solving \cref{obj} using a random walk SGD method \cite{robbins1951stochastic, bertsekas2011incremental, nemirovski2009robust}, that is, with the initial model $x^0$, starting  node $v_0$, repeat the following steps: at iteration $t$,  node $v_t$ updates the model using the stochastic gradient $\hat{g}_{v_t}$ calculated based on the local data, and then passes the model to one of its neighbors  randomly chosen according to a transition matrix $P$.

We focus on the effect of designing   $P$ on the algorithm's convergence. Three   choices of   $P$ have been studied. \begin{enumerate}
    \item  The simplest approach is to choose the next node $u$ is chosen uniformly at random: $P(v,u)=\ \frac{1}{\deg(v)}$. However, the stationary distribution of this  RW is proportional to the nodes' degrees.
    \item The most extensively investigated approach is to employ the \mh\ (MH) algorithm  \cite{metropolis1953equation,hastings1970monte} to construct $P$ with the goal of achieving a uniform stationary distribution\cite{johansson2007simple}: $P(v,u)=\frac{1}{\deg(v)}\min\{1, \frac{\deg(v) }{\deg(u) }\}, u\neq v,(u,v)\in E$. This design tries to mimic the vanilla centralized SGD that samples the data uniformly.
    \item For a general desired stationary distribution $\pi$ on the nodes, use the following transition probability obtained from MH \cite{ayache2021private}: \[P(v,u)=\min\left\{\frac{1}{\deg(v)}, \frac{\pi_u}{\deg(u) \pi_v}\right\}, u\neq v,(u,v)\in E.\] 
\end{enumerate}

We are interested in the last option, where the desired distribution $\pi$  is set to be the importance sampling distribution.
 In centralized learning, sampling the data according to their ``importance'' may speed up the convergence \cite{needell2014stochastic}. We show in this work that using the \mh\ algorithm to implement importance sampling in decentralized learning may cause a phenomenon we term the entrapment problem. 
 Namely, the random walk may become entrapped in specific nodes or regions of the graph for an extended duration, thereby slowing down the convergence rate.
 We propose a new design strategy  based on perturbing the \mh\ transition probability $P$ with L\'{e}vy-like jumps \cite{riascos2012long} to overcome  entrapment and show that it can speed up the convergence.
\subsection{Previous Work}
The original work by  \cite{johansson2007simple,johansson2010randomized,lopes2007incremental} marked the initial exploration of random walk learning with the (sub-)gradient method. Subsequently, the work of \cite{ram2009incremental} extended the results to accommodate changing topology networks. 
Utilization of curvature information to accelerate the convergence rate of random walk SGD was studied in \cite{wai2018sucag}. Meanwhile, random walk SGD with adaptive step sizes was investigated in \cite{sun2022adaptive}.
Ergodic sampling for the mirror descent method was explored in \cite{duchi2012ergodic}. Under Markovian sampling, non-convex results for SGD were presented in \cite{sun2018markov}, the AdaGrad method was examined in \cite{dorfman2022adapting}, variance reduction methods applicable to non-convex cases are investigated by \cite{even2023stochastic}.

Another direction relevant to this work is importance sampling. Needell et al. showed in \cite{needell2014stochastic} that sampling the data proportional to the gradient Lipschitz constant 
of the local loss function can speed up the convergence rate when the data is heterogeneous.   The work in  \cite{zhao2015stochastic} connected importance sampling with minimizing the variance of the gradient estimator. Importance sampling for minibatches was studied in \cite{csiba2018importance}. All these works focused on centralized scenarios and did not address the decentralized case we focus on here. 

All the aforementioned work on random walk learning primarily studied scenarios where the random walk's stationary distribution is uniform across the nodes. The work in \cite{ayache2021private,ayache2023walk} went beyond uniform sampling using importance sampling and multi-armed bandits.

\subsection{Contributions}
 
We investigate the impact of designing the transition probability of the random walk on the convergence properties of the random walk learning algorithm. We start by implementing importance sampling in a decentralized learning framework via the \mh\ algorithm. Subsequently, we show that when the data is heterogeneous, and the network is not well-connected, the random walk governed by \mh\ transition may become entrapped at certain ``important'' nodes. We call this phenomenon the entrapment problem. This entrapment phenomenon will force the model to be biased toward the local data and marginalize the updates. 
To mitigate this issue, we propose a novel algorithm, the \mh\ with L\'{e}vy Jump (\algname), which incorporates random perturbations to help the random walk avoid getting entrapped. 
We then analyze the convergence rate and error gap of the \algname\ algorithm, supplementing our theoretical findings with simulations to validate our results.

\subsection{Organization}
 The rest of the paper is organized as follows: Section II introduces the problem setting. Sections III and IV introduce the decentralized way of implementing importance sampling and the entrapment problem. Our proposed algorithm, MHLJ, to overcome the entrapment problem and the simulation results are introduced in Section V. Finally, we give our theoretical convergence result of MHLJ in Section VI.  The complete proof of the theoretical convergence result and the simulation settings can be found in the  Appendix of an extended version of this paper\footnote{https://github.com/ZonghongLiu/ISIT2024-Entrapment-Extended}.
   \section{Problem Setting}
    \subsection{Network and Objective Function}
    We consider a communication network represented by a connected graph $G=(V, E)$, where $V$ is the set of nodes, and $E\subseteq V\times V$ represents the communication links between nodes. Nodes that are connected can communicate with each other. We assume that each node in the graph has a self-loop. Each node $v$ of the network has its local data $x\in R^d$,  which induces a local loss function $f_v(x)$. The goal is to find a decentralized algorithm to solve \cref{obj} using only local communications without the help of a central server.
 The objective function to minimize can be expressed as follows:
\begin{align}
    f(x)=\frac{1}{|V|}\sum_{v\in V}f_v(x),\ x\in \mathbb{R}^d.
\end{align}
\subsection{Data Heterogeneity}
We will work under the Lipschitz smooth assumption: 
\begin{definition}
    A function $f(x)$ is L-smooth if 
    \begin{align*}
       \lVert\nabla f(x)-\nabla f(y)\rVert \leq L\lVert x-y\rVert, \textit{ for all }x,y\in \textit{dom}(f),
    \end{align*}
      where $L$ is the gradient Lipschitz constant of the function.
    \end{definition}
    For example, in linear regression $f_v(x)=\frac{1}{2}\lVert y_v-x^TA_v\rVert^2$ one can set $L_v=\lVert A_v\rVert^2$, and in Logistic regression $f_v(x)=y_v x^T A_v-\log(1+e^{x^T A_v})$ one can set $L_v$  can be chosen as $\frac{1}{4}\lVert A_v\rVert^2$, where $(A_v,y_v)$ is the local data stored at node $v$.
 
     We are interested in the scenario where the data owned by the nodes is heterogeneous, i.e., not sampled from identical distributions. We will look at the gradient Lipschitz constants $L_v$ of the local loss functions $f_v$ as a proxy for heterogeneity.    We denote 
     $L_{\max}=\max\left\{L_v|v\in V\right\}$, $L_{\min}=\min\left\{L_v|v\in V\right\}$, and $\Bar{L}=\frac{1}{|v|}\sum_{v\in V}L_v$.  
     In particular, we consider the following heterogeneous scheme:
    \begin{align}
        L_{\min}\approx\Bar{L}\leq L_{\max}.\label{hetedata}
    \end{align}
  In this case, importance sampling consisting of sampling data proportional to the local gradient Lipschitz constant can lead to a speed-up in convergence \cite{needell2014stochastic} .
\subsection{Random Walk Learning}
We want to design a decentralized algorithm that solves \cref{obj} via a random walk.  
A random walk algorithm  for decentralized optimization \cite{ayache2021private} with a given transition probability matrix $P$ consists of the following steps:
\begin{enumerate}
    \item Start from a randomly selected node $v_0$, with the currently visited initial model $x^0$;
    \item At iteration $t$,  $v_t$ updates the model using the stochastic gradient $\hat{g}_{v_t}$ calculated based on the local data:
    \begin{align}
    x^{t+1}=x^t-\gamma_t\hat{g}_{v_t}(x^t),\label{sgd}
\end{align}  
\item Node $v_t$ randomly chooses one of its neighbors (including itself) as $v_{t+1}$, according to a distribution $P(v_t,\cdot)$.
\item  Node $v_t$ passes the model $x^{t+1}$ to node $v_{t+1}$.
\end{enumerate}
The algorithm runs steps 2), 3), and 4) iteratively for a given number of iterations $T$. The model passed among the nodes and their neighbors can be seen as a time-homogeneous random walk on the graph $G$ with transition matrix $P$. We assume that  $P$ is aperiodic and recurrent. Therefore, the random walk is ergodic and converges to a stationary distribution $\pi$.

\section{Importance Sampling}
Our main objective is to design the transition matrix $P$ of the random walk to speed up the convergence of the learning algorithm. Our approach is to mimic centralized importance sampling, which has been shown to improve convergence in certain regimes \cite{zhao2015stochastic,needell2014stochastic}.  The main challenge is that the data sampled by the random walk is not i.i.d anymore, but is governed by a Markovian dependency imposed by the graph.


\subsection{Importance Sampling in Centralized Learning}
Importance sampling has been mainly studied in the centralized setting. In the standard SGD algorithm, the data is sampled uniformly. Importance sampling goes beyond the uniform distribution and samples the data based on a certain measure of importance, still in  i.i.d. fashion. 
Of particular importance to our work here is the work of Needell et al. \cite{needell2014stochastic}, where it was shown that weighted sampling in a centralized setting can speed up the convergence of SGD. It was proposed to use the gradient Lipschitz constant of the local loss function   $L_i$ as the importance of data $x_i$,  and to sample the data  proportional to its importance, i.e., according to the following distribution: 
\begin{align}
    \pi_{IS}(i)\coloneqq\frac{L_i}{\sum_{i=1}^{N}L_i}, \label{isdist}
\end{align}
where $\pi_{IS}$ is defined to be the importance sampling distribution. The following convergence rates of SGD  for Lipschitz smooth and strongly convex objective functions were shown in \cite[Theorem~2.1]{needell2014stochastic}:
\begin{enumerate}
    \item Uniform Sampling: $\Tilde{\mathcal{O}}(\frac{L_{\max}}{T})$;
    \item Importance Sampling: $\Tilde{\mathcal{O}}(\frac{\Bar{L}^2}{L_{\min}T})$,
\end{enumerate}
where $L_{\max}$, $L_{\min}$,  and $\Bar{L}$ are the maximum, minimum, and average of local gradient Lipschitz constants, respectively. 
From the convergence rate, we see that when \cref{hetedata} holds, i.e., when the gap between $L_{\max}$ and $\Bar{L}$ is significant, and $L_{\min}$ is close to $\Bar{L}$, sampling according to the importance distribution in \cref{isdist} will speed up the convergence  of SGD.
\subsection{Importance Sampling in Decentralized Learning}
\begin{figure}[t]
\centering
\subfloat[]{\begin{tikzpicture}[node distance=2cm, scale=0.5]
    \tikzstyle{node} = [circle, draw, minimum size=1.5em]

    \node[node, fill=green!20!gray,scale=0.75] (A) at (18:2cm) {2};
    \node[node, fill=cyan!50!gray,scale=0.75] (B) at (90:2cm) {1};
    \node[node, fill=green!20!gray,scale=0.75] (C) at (162:2cm) {5};
    \node[node, fill=green!20!gray,scale=0.75] (D) at (234:2cm) {4};
    \node[node, fill=green!20!gray,scale=0.75] (E) at (306:2cm) {3};

    \node [above right=0.01cm and 0.01cm of A] {\small $L_2\approx 1$};
    \node [above=0.01cm of B] {\small $L_1\approx 100$};
    \node [above left=0.01cm and 0.01cm of C] {\small$L_3\approx 1$};
    \node [below left=0.001cm and 0.01cm of D]  {\small$L_4\approx 1$};
    \node [below right=0.01cm and 0.01cm of E]  {\small$L_5\approx 1$};

    \draw[-] (A) to  (B);
    \draw[-] (B) to  (C);
    \draw[-] (C) to  (D);
    \draw[-] (D) to  (E);
    \draw[-] (E) to (A);
\end{tikzpicture}}
\subfloat[]{
    \begin{tikzpicture}[node distance=2cm, scale=0.5]
    \tikzstyle{node} = [circle, draw, minimum size=1.5em]

    \node[node, fill=green!20!gray,scale=0.75] (A) at (18:2cm) {2};
    \node[node, fill=cyan!50!gray,scale=0.75] (B) at (90:2cm) {1};
    \node[node, fill=green!20!gray,scale=0.75] (C) at (162:2cm) {5};
    \node[node, fill=green!20!gray,scale=0.75] (D) at (234:2cm) {4};
    \node[node, fill=green!20!gray,scale=0.75] (E) at (306:2cm) {3};

    \node at (22:2.5cm)  {};
    \node at (90:3cm)  {};
    \node at (157:2.5cm)  {};
    \node at (234:2.5cm)  {};
    \node at (306:2.5cm)  {};

    \draw[->] (A) to[bend left]  node[]{ $\approx \frac{1}{2}$}  (B);
    \draw[->] (B) to[bend left] (C);
    \draw[->] (C) to[bend left] (D);
    \draw[->] (D) to[bend left] (E);
    \draw[->] (E) to[bend left]  (A);

    \draw[->] (B) to[bend left] node[above right] { $\approx\frac{1}{102}$} (A);
    \draw[->] (C) to[bend left] (B);
    \draw[->] (D) to[bend left] (C);
    \draw[->] (E) to[bend left] (D);
    \draw[->] (A) to[bend left] node[right] { $\approx\frac{1}{2}$}(E);

    \draw[->] (A) to[loop right]node[above] { $\approx0$}(A);
    \draw[->] (B) to[loop above] node[above] { $\approx\frac{50}{51}$}(B);
    \draw[->] (C) to[loop left] (C);
    \draw[->] (D) to[loop left] (D);
    \draw[->] (E) to[loop right] (E);
\end{tikzpicture}%
 }
\caption{\small (a) An example of ring topology with five nodes that may cause the entrapment issue.  
(b) In the Markov chain representation of the random walk on the graph in (a).} 
\label{5ring}
\vspace{-0.4cm} 
\end{figure}
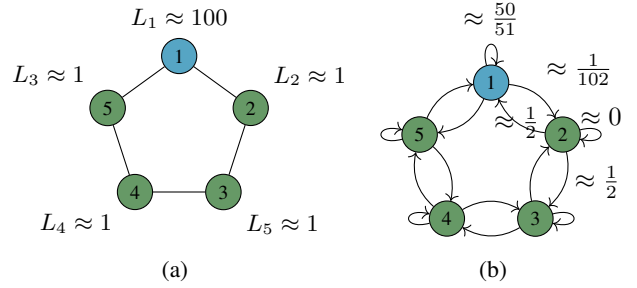

 In a graph-based decentralized setting, there is no central server to implement importance sampling. 
 Ayache et al. \cite{ayache2021private} proposed importance sampling in random walk learning by designing the transition probability  $P$  to achieve a desired stationary distribution $\pi=\pi_{IS}$, which is proportional to $L_i$ as in \cref{isdist}, via \mh\  algorithm. Given a distribution $\pi$, the MH algorithm allows  designing a transition matrix $P$ that has $\pi$ as its stationary distribution:
\begin{align}
            P (i,j)=\left\{\begin{array}{cc}
          Q(i,j)\min\{1, \frac{\pi(j)Q(j,i)}{\pi(i)Q(i,j)}\},& i\neq j,\\
          1-\sum_{k:(i,k)\in E}P(i,k),&i=j,
     \end{array}
     \right. 
        \end{align}
        where $Q$ is any proper transition probability that satisfies the graph structure, i.e., $Q(i,j)=0$ if $(i,j)\notin E$, $Q^k(i,j)>0$ for some  $k$ if there is a path from $i$ to $j$. For example, we can take $Q$ as the simple random walk on the graph, i.e., the neighbors are selected uniformly $Q(i,j)=\frac{1}{\deg(i)}, \forall(i,j)\in E$.

To mimic  importance sampling, i.e., $\pi(i)\propto L_i$,  the transition probability matrix can be chosen to be \cite{ayache2021private}:
 \begin{align}
     P_{IS}(i,j)=\left\{\begin{array}{cc}
          \frac{1}{\deg(i)}\min\{1, \frac{\deg(i) L_j}{\deg(j) L_i}\},& i\neq j,\\
          1-\sum_{k:(i,k)\in E}P(i,k),&i= j.
     \end{array}
     \right. \label{mhtrans}
 \end{align}
\section{The Entrapment Problem}
In certain cases, the \mh\ importance sampling transition given by {\cref{mhtrans}
}
can lead to a degradation in the convergence rate. We show that when the data is heterogeneous, and the graph is not ``well-connected", the random walk moving according to {\cref{mhtrans}} may get entrapped in a local area of the graph, leading to a slowdown in convergence.

We will illustrate our ideas using the example of a ring network with heterogeneous data. \cref{5ring}.a gives such an example where the data are stored over a ring network with $5$ nodes. Here, node $1$  stores the data set that has a much larger gradient Lipschitz constant. 
We show that in this case, the convergence rate of \mh\ importance sampling is dramatically slowed down, as shown in  \cref{mainsim}. The reason is that the random walk is getting entrapped on the ``important'' nodes, i.e., nodes holding data with large $L_i$'s.  
 This forces the algorithm to update the model using the same data a large number of times, pushing the model to converge to the local optimum, thus slowing down convergence. 

To understand the cause of the entrapment problem, notice that the transition probability $P_{IS}$ of the random walk  given in \cref{mhtrans}  satisfies the detailed balanced condition $\pi(i)P_{IS}(i,j)=\pi(j)P_{IS}(j,i)$ \cite{levin2017markov}, which in our case leads to
\begin{align}
     L_i/L_j= P_{IS}(j,i)/P_{IS}(i,j).
\end{align}
Therefore, when a node has much larger local gradient Lipschitz constants than its neighbors', and the graph is sparse, the probability of leaving this node is very small.

We observe that the entrapment problem does not occur only in the ring network but also in other ``sparse'' networks, like $2D$-grids and Watts-Strogatz random graphs.
\section{\algname\ Algorithm}
We propose a new algorithm, Metropolis-Hastings with L\'{e}vy Jumps (\algname), to solve the entrapment problem. The main idea consists of perturbing the Metropolis-Hastings transition probability in \cref{mhtrans} by adding random jumps to escape   a local entrapment.
\begin{algorithm}[
b]
 \caption{Importance Sampling using Metropolis-Hastings with L\'{e}vy Jumps (\algname)}
 \begin{algorithmic}[1]
 \renewcommand{\algorithmicrequire}{\textbf{Input:}}
 \renewcommand{\algorithmicensure}{\textbf{Output:}}
 \REQUIRE $G=(V,E)$, $L_v$ for $v\in V$, $P_{IS}$,  $\gamma$, $T$, $p_{J}$, $p_d$, $r$
 \ENSURE  $x^{T}$
 \\ \textit{Initialisation}: $x^0$, $v_0$,
 \\ \FOR {$t=0,1,T-1$}
  \STATE $x^{t+1}=x^t-\gamma \frac{\Bar{L}}{L_{v_t}}\nabla f_{v_t}(x^t)$\\
  \STATE $J\sim\textit{Ber}(p_{J})$ 
  \IF {$J = 0$}
  \STATE $v_{t+1}$ $\sim$ $P_{IS} (v_{t},\cdot)$
  \ELSE  
  \STATE $d\sim \mathsf{TruncGeom}(p_d, r)$ 
  \WHILE{$d \geq 0$}
  \STATE $v_{t+1}\sim \textit{Unif}(\mathcal{N}_{v_t})$\\
$v_{t} = v_{t+1}$ \\
$d=d-1$
  \ENDWHILE
  \ENDIF  
  \ENDFOR
 \RETURN $x^T$ 
 \end{algorithmic} 
 \end{algorithm}
 The added jump requires no global information on the graph. Each step of the jump requires only local structure information, i.e., the neighbors of the current node. The details are described in Algorithm 1, where $(p_{J},p_d,r)$ are the parameters of the L\'{e}vy jumps, and $\mathcal{N}_v$ is the neighbor set of node $v$.
 \begin{figure}[t]
 {
    \includegraphics[width = 0.48\textwidth]{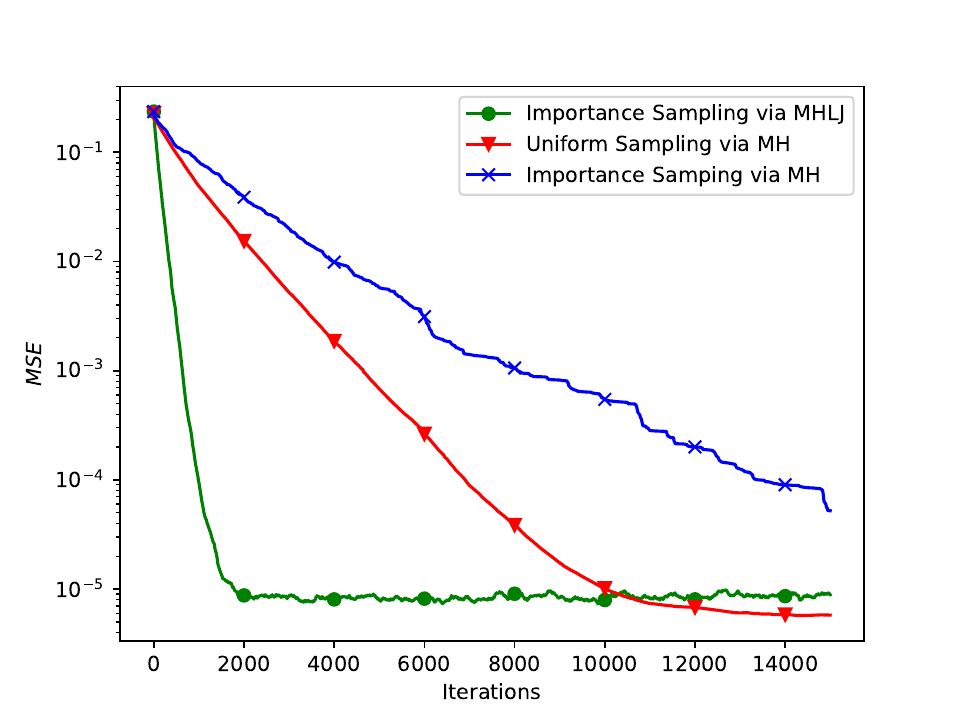}
 }
       \caption{\small Linear regression model $y =A  x +\epsilon$ trained on a synthetic heterogeneous data set over a ring network with 1000 nodes. We compare the uniform sampling, importance sampling, and our Algorithm \algname. The $y$-axis is the mean square error (MSE), i.e., $ \sum_{v\in V}\norm{y_v - A_v\hat x}^2/|V|$. The $x$-axis is the number of iterations with SGD updates, i.e., the number of times  \cref{sgd} is called. We generate the data $A_v$ on node $v$ with $A_v\overset{\mathrm{i.i.d.}}{\sim} N(0,\sigma^2\mathbb{I}_{10})$, where $\sigma^2$ takes value $1$ with probability $p=0.998$ and $100$ with probability $p=0.002$. The noise is generated from $\epsilon\overset{\mathrm{i.i.d.}}{\sim}N(0,1)$. We use the hyper-parameters: $(p_J,p_d,r)=(0.1,0.5,3)$.} 
       \vspace{-0.5cm} 
  \label{mainsim}
\end{figure}
 We compare the performance of Uniform sampling via MH, Importance sampling via MH, and Importance sampling via \algname\ in \cref{mainsim} for the ring network with $1000$ nodes. 
 The simulation results show the following:
\begin{enumerate}
    \item[i.]  \algname\ can break the entrapment and significantly speed up the convergence rate.
    \item[ii.] \algname\ exhibits asymptotically an error gap that we will later explain in our theoretical analysis. 
\end{enumerate}

In \algname\ the random walk determines its next step after each update. Specifically, it either executes a Lévy jump with a probability of $p_J$ or adheres to the \mh\ rule with a probability of $1 - p_J$. 

\textit{L\'{e}vy jump:} When the random walk makes a jump: (a) The random walk chooses how far it should jump. The jumping distance $d$ is sampled from a truncated Geometric (TrunGeom) distribution defined by $P(D=d)=\frac{p_d(1-p_d)^{d-1}}{1-(1-p_d)^{r}}\mathbb{I}\{d\leq r\}$. 
(b) Once the distance $d$ is determined, the model undergoes $d$ consecutive transfers between nodes, wherein it is passed to a uniformly selected neighboring node $d$ times in succession without undergoing any updates. 
The simple random walk strategy employed during the jumps is deliberately designed to disrupt the detailed balance condition, thereby enabling the random walk to escape the entrapping region. As a consequence, the sampling distribution of nodes deviates from the desired importance distribution defined in \cref{isdist}, resulting in an error gap, which will also appear in our convergence result presented in \cref{thm}.

Algorithm 1 induces a time-homogeneous random walk with transition matrix $P$. We view this random walk as a Metropolis-Hastings random walk (with transition matrix $P_I$ defined in \cref{mhtrans}) perturbed by L\'{e}vy jumps, i.e., 
\begin{align*}
    P&=(1-p_{J}) P_I + p_{J}P_{\textit{L\'{e}vy}},\textit{ where}\\
    P_{\textit{L\'{e}vy}}&= \sum_{i=1}^{r}\frac{p_d(1-p_d)^{i-1}}{1-(1-p_d)^{r}}\diag\{A_G^i \mathbf{1}\}^{-1}A_G^i,
    \vspace{-0.5cm}
\end{align*}
where $A_G$ is the adjacency matrix of the given graph $G$.
 The resulting stationary distribution $\pi$ is thus no longer $\pi_{IS}(v)=\frac{L_v}{\sum_{v\in V}L_v}$ but a perturbed version of it. 
 
\begin{remark}[Computation v.s. Communication overheads of \algname]  Each iteration in \algname\ ($x$-axis in \cref{mainsim}) corresponds to one gradient decent update according to \cref{sgd}. \cref{mainsim} shows that \algname\ saves on computation cost since it requires less updates to achieve a given accuracy. However, by adding jumps, we actually admit transitions without updates, which leads to an increase in the communication overhead. For each update, the expected number of transitions (node visits) required can be bounded by
\begin{align*}
    (1-p_J)\cdot 1+p_J\E{d}\leq 1+p_J(\frac{1}{p_d}-1).
\end{align*}
In our example, this upper bound is equal to $1.1$, i.e., at most $10\%$ increase in the average communication cost in our example.
\end{remark}

 \section{Convergence Result}
Now, we give our theoretical convergence result.
\begin{thrm}[Convergence of Algorithm \algname]
    Suppose that each local loss function $f_v$ is  $L_v$-smooth and $\mu$-strongly convex, and $\norm{\nabla f_v(x^*)}^2\leq \sigma_*^2,\ \forall v\in V$, then for $\gamma<\min\{\frac{1}{\Bar{L}},\frac{1}{T\mu}\ln{T\frac{\norm{x^0-x^*}^2\mu^2}{\tau_{mix}\sigma^2_*\Bar{L}}}\}$, the output of Algorithm 1 after $T$ iterations $x^T$
    satisfies:\footnote{$\Tilde{\mathcal{O}}$ hides logarithmic factors.}
    \begin{align}
        \hspace{-0.2cm}\mathbb{E}\norm{x^T-x^*}^2\leq \Tilde{\mathcal{O}}\left(\frac{\Bar{L}^2\tau_{mix} \sigma_*^2}{L_{\min}T}\right)+\mathcal{O}\left({ p_{J}^2\norm{P_{IS}-P_{\textit{L\'{e}vy}}}^2_1 }\right)\label{mainresult}
    \end{align}
    where $\tau_{mix}$ is the mixing time of $P=P_{IS}-p_{J}(P_{IS}-P_{\textit{L\'{e}vy}})$, and $\Bar{L}=\sum_{v\in V}L_v/|V|$.  
    \label{thm}
\end{thrm}

The first term in \cref{mainresult} implies that the algorithm converges with a sub-linear rate. Here,  $\tau_{mix}$ is the mixing time \cite{levin2017markov} of the random walk and represents the effect of sampling dependency induced by the graph topology. Also, note that $\tau_{mix}$ is smaller than its \mh\ counterpart because making jumps makes the graph better connected. The second term describes the error gap caused by the jumps.
The choice of $p_J$ creates a trade-off between the speed with which the random walk can escape from the entrapment and the magnitude of the error gap expressed in the second term of \cref{mainresult}.  When the value of $p_J$ is small, the random walk experiences difficulty escaping the entrapment, resulting in a slow convergence; conversely, a large value of $p_J$ yields a more substantial error gap.
As for $\norm{P_I-P_{\textit{L\'{e}vy}}}_1$, its value  depends on the graph and the gradient Lipschitz constants, and can be upper bounded by $n^2$. In practice, the error gap can be made arbitrarily small by decreasing $p_J$ as the number of iterations increases.

The proof of \cref{thm} presents two challenges compared to   the standard proof of SGD:
\begin{enumerate}
    \item The stochastic gradient $\nabla f_{v_t}(x^t)$ used in each update step  is not an unbiased estimator of the true gradient due to the graph topology, i.e., $\E{\nabla f_{v_{t}}(x^{t})\mid v_{t-1}}\neq \nabla f(x^{t})$. Thus, each step is not a descent step in expectation as in standard SGD.
    \item The detailed balance equation is violated by the added L\'{e}vy jumps, causing the expectation with respect to the stationary distribution to be   also biased, i.e., $\mathbb{E}_{\pi_I}[\nabla f_v(x^*)]\neq 0$. This  breaks the first order optimality condition .
\end{enumerate}

To address the first challenge, we use  an auxiliary sequence $\{y^t\}_{t=1}^T$  to bound $\norm{x^T-x^*}$ without relying on  the conditional unbiasedness  of the gradient estimate. 
This proof technique was first introduced  in \cite{mishchenko2020random} to study the random reshuffling method,  and then used for the proof of Markovian SGD in \cite{even2023stochastic}. Namely, we construct $\{y^t\}_{t=1}^T$ by letting:
\begin{align}
    y^{t+1}=y^t-\gamma \frac{\Bar{L}}{L_{v_t}} \nabla f_{v_t}(x^*).\label{auxila}
\end{align}
The following lemma controls the distance between $x^t$ and $y^t$ and is adapted   from Lemma 9  in \cite{even2023stochastic} to incorporate the smoothness constants.
\begin{lem}
    For any $\{y^t\}_{t=0}^T$ satisfies \cref{auxila}, we have 
    \begin{align*}
    \norm{x^{t+1}-y^{t+1}}^2 \leq (1-\gamma \mu)\norm{x^t-y^t}^2+\gamma \Bar{L}\norm{y^t-x^*}^2.
    \end{align*}\label{boundcumu1}
    \vspace{-0.5cm}
    \label{lem1}
\end{lem}
By setting $y^T=x^*$, we can upper bound $\norm{x^t-x^*}^2$ by
\begin{multline}
        \E{\norm{x^T-x^*}^2}
        \leq 2(1-\gamma\mu)^T\norm{x^0-x^*}^2\\
        +3\gamma^3\Bar{L}\sum_{t\leq T}(1-\gamma\mu)^{T-t}\underbrace{\E{\left\Vert\sum_{t\leq s\leq T}\frac{L_{v_s}}{\Bar{L}}\nabla f_{v_s}(x^*)\right\Vert^2}}_{\textit{Accumulated error term}}.\label{cumsum}
\end{multline}
In the case of an MH random walk with no jumps, the accumulated error term in \cref{cumsum} should converges to zero as $T\to \infty$ due to ergodicity. To address the second challenge,   we prove in  \cref{lem2} an upper bound on  the accumulated error term for \algname. 

\begin{lem}
  For $1\leq s\leq t\leq T$, we have 
    \begin{align*}
        &\E{\left\Vert\sum_{i=s}^t\frac{L_{v }}{\Bar{L}}\nabla f_{v_i}(x^*)\right\Vert^2}\\&\leq (t-s)C\tau_{mix}\sigma_*^2+2(t-s)^2p_J^2\norm{P_{IS}-P_{\textit{L\'{e}vy}}}_{1}^2\sigma^2_* \left(\frac{\Bar{L}}{L_{\min}}\right)^2.
    \end{align*}
   \label{lem2}
\end{lem}
 Lemmas \cref{boundcumu1}  and \cref{lem2} serve as essential blocks for completing the proof of Theorem \cref{thm}.

\IEEEtriggeratref{20}
\nocite{mao2020walkman,hendrikx2023principled,jiang2023facegroup, naor2020concentration,seneta1988perturbation,nedic2018network}
\bibliographystyle{IEEEtran}
\bibliography{bibliofile}

\IEEEtriggeratref{4}

\newpage
\appendix
\section{Proofs}
In this section, we prove the convergence results for \algname \ for $L$-Lipschitz and strongly convex objective functions, and show the simulation results for our work. 
\subsection{Assumptions and Formal Statement of the Main Theorem}
We make the following general assumptions on the local loss functions.

\textit{Assumption 1.}
    Local Lipschitz smoothness:
    \begin{align*}
        \norm{\nabla f_v(x)-\nabla f_v(y)}\leq L_v\norm{x-y},\ \forall x,y\in\mathcal{X}, \forall v\in V.
    \end{align*}
    
\textit{Assumption 2.}
    Local strong convexity:
    \begin{align*}
        f_v(y)-f_v(x)\geq \ip{\nabla f_v(x)}{y-x}+\mu \norm{y-x}^2,\ \forall v\in V.
    \end{align*}
    
\textit{Assumption 3.}
    Bounded norm of the local gradient at the global optimum:
    \begin{align*}
        \norm{\nabla f_v(x^*)}^2\leq \sigma_*^2,\ \forall v\in V,
    \end{align*}
    where $x^*$ is the minimizer of (2).

In Algorithm 1, the update is given by 
\begin{align}
    x^{t+1}=x^t-\gamma \nabla \left(\frac{\Bar{L}}{L_{v_t}}f_{v_t}(x^t)\right).\label{update}
\end{align}
Where the transition probability can be seen as the \mh\ transition probability $P_{IS}$ perturbed by a L\'{e}vy jumping transition matrix $P_{\textit{L\'{e}vy}}$, whose closed form is know
\begin{align*}
    P_{\textit{L\'{e}vy}}= \sum_{i=1}^{r}\frac{p_d(1-p_d)^{i-1}}{1-(1-p_d)^{r}}\diag\{A_G^i \mathbf{1}\}^{-1}A_G^i
\end{align*}
\subsection{Intermediate Lemmas}
To prove Theorem 1, we follow the idea in \cite{even2023stochastic}. We construct an auxiliary sequence $\{y^t\}_{t=0}^T$ by:
\begin{align}
   y^{t}&=y^{t-1}-\gamma \frac{\Bar{L}}{L_{v_{t-1}}}\nabla f_{v_{t-1}}(x^*).\label{auxiliary}
\end{align}
Note that, given a value of any $y^s,s\in[T]$, the sequence $\{y^t\}_{t=0}^T$ is determined by $\{v_t\}_{t=0}^T$. By setting $y^T=x^*$, we can bound the distance between $x^T$ and $x^*$.

The following lemma from \cite{even2023stochastic} controls the distance between the two sequences. 
\begin{lem}
    If $x^t$ is generated from (\ref{update}), for any $\{y^t\}_{t=0}^T$ satisfies (\ref{auxiliary}), we have 
    \begin{align*}
    \norm{x^{t+1}-y^{t+1}}^2&\leq (1-\gamma \mu)\norm{x^t-y^t}^2\\
    &+\gamma \Bar{L}\norm{y^t-x^*}^2.
    \end{align*}\label{boundcumu}
\end{lem}
Our main contribution is the following Lemma on bounding the accumulated $L_2$-norm of $\nabla f_{v_t}(x^*)$ when the random walk moves according to a perturbed transition probability. 
\begin{lem}
Let $P=(1-p_J)P_{IS}+p_{J}P_{\textit{L\'{e}vy}}$.  For $1\leq s\leq t\leq T$, we have 
    \begin{align*}
        &\E{\left\Vert\sum_{i=s}^tw(v_i)\nabla f_{v_i}(x^*)\right\Vert^2}\\&\leq (t-s)C\tau_{mix}\sigma_*^2+2(t-s)^2p_J^2\norm{P_{IS}-P_{\textit{L\'{e}vy}}}_{1}^2\sigma^2_* w_{max}^2,
    \end{align*}
    where $\tau_{mix}$ is the mixing time of chain $P$, $\nu$ is the stationary distribution of $P$, $\pi(v)=\frac{L_v}{\sum L_v}$, $w(v)=\frac{\Bar{L}}{L_v}$, and $w_{max} =\frac{\Bar{L}}{L_{min}}$. 
\end{lem}
\begin{proof}[Proof of Lemma 2]
\begin{align*}
    &\frac{1}{(t-s)^2}\E{\left\Vert\sum_{i=s}^t w(v_i)\nabla f_{v_i}(x^*)\right\Vert^2}\\
    &=\E{\left\Vert\frac{1}{t-s}\sum_{i=s}^t w(v_i)\nabla f_{v_i}(x^*)\right\Vert^2}\\
    &\overset{(a)}{\leq}2\E{\left\Vert\frac{\sum_{i=s}^t w(v_i)\nabla f_{v_i}(x^*)}{t-s}
    -{\mathbb{E}_{\nu}\lbrack w(v)\nabla f_{v}(x^*)\rbrack}\right\Vert^2}\\
    &+2\left\Vert\mathbb{E}_{\nu}\lbrack w(v)\nabla f_{v}(x^*)\rbrack\right\Vert^2\\
    &\overset{(b)}{\leq} \frac{2\tau_{mix}\sigma^2_*}{c (t-s)}+2\left\Vert\mathbb{E}_{\nu}\lbrack w(v)\nabla f_{v}(x^*)\rbrack\right\Vert^2\\
    &\leq \frac{C}{t-s}\tau_{mix}\sigma^2_*+2\norm{\nu-\pi}_{TV}^2\sigma^2_* w_{max}^2,
    \end{align*}
    Where (a) follows from $\norm{x+y}^2\leq 2\norm{x}^2+2\norm{y}^2$, (b) follows from \cite{naor2020concentration}.\footnote{$\mathbb{E}\left\Vert\frac{1}{n}\sum_{i=1}^nf(X_i)-\mathbb{E}_{\pi}{f(X)}\right\Vert^2\leq \frac{2C\tau_{mix}}{n}\max_{x\in\mathcal{X}} \norm{f(x)}^2$, where $X_i$ are sampled from Markov chain on $\mathcal{X}$ with stationary distribution $\pi$. We can run the random walk for sufficient long time such that the chain converges to stationary first, which takes only $\mathcal{O}(\tau_{mix})$ time and won't affect the order of convergence.}, the last inequality uses $\mathbb{E}_{v\sim\pi}\left[w(v)\nabla f(x^*)\right ]=0$ and 
    \begin{align*}
        &\norm{\mathbb{E}_{v\sim\pi}[w(v)\nabla f_{v}(x^*)]-\mathbb{E}_{v\sim\nu}[w(v)\nabla f_{v}(x^*)]}^2\\
        &\leq \norm{\pi-\nu}_{TV}^2\cdot \max_v \norm{\nabla f_v(x^*)}^2 w_{max}^2.
    \end{align*}
    Further, upper bound the total variation distance by perturbation bound $\norm{\pi-\Tilde{\pi}}_{TV}\leq C\norm{P-\Tilde{P}}_1$ from \cite{seneta1988perturbation},
    then multiplying $(t-s)^2$ on both sides, we have the desired result.
\end{proof}
\subsection{Formal Proof}
Now, we are ready to prove Theorem 1.
\begin{proof}[Proof of Theorem 1]
Set $y^T=x^*$, from Lemma 1, we have
    \begin{align*}
    \norm{x^T-y^T}^2&\leq (1-\gamma\mu)^T\norm{x^0-y^0}^2\\&+\gamma \Bar{L}\sum_{t\leq T}(1-\gamma\mu)^{T-t}\norm{y^t-x^*}^2.
    \end{align*}
    Note that
    \begin{align*}
        y^0&=x^*+\gamma\sum_{t\leq T}w(v_t)\nabla f_{v_t}(x^*),\\
        y^t&=x^*+\gamma\sum_{t\leq s\leq T}w(v_s)\nabla f_{v_s}(x^*),
    \end{align*}
    we can upper bound:
    \begin{align*}
        &\E{\norm{x^T-x^*}^2}
        \leq 2(1-\gamma\mu)^T\norm{x^0-x^*}^2\\
        &+3\gamma^3\Bar{L}\sum_{t\leq T}(1-\gamma\mu)^{T-t}\E{\left\Vert\sum_{t\leq s\leq T}w(v_s)\nabla f_{v_s}(x^*)\right\Vert^2}.
        \end{align*}
        The third term is then upper bounded by Lemma 2 we have
        \begin{align}
            &\E{\norm{x^T-x^*}^2}\leq 2(1-\gamma\mu)^T\norm{x^0-x^*}^2\\ &+C_1\gamma^3\Bar{L}\sum_{t\leq T}(1-\gamma\mu)^{T-t}(T-t)\color{black}\tau_{mix}\sigma^2_*\\
        &+C_2p_J^2\norm{P_{IS}-P_{\textit{L\'{e}vy}}}^2_{1}\sigma^2_*\gamma^3\frac{\Bar{L}^3}{L_{min}^2}\sum_{t\leq T}(1-\gamma\mu)^{T-t}(T-t)^2.
        \end{align}
        Finally, use the numerical inequality $\sum_{t\leq T}(1-x)^t t\leq 1/x^2$ and $\sum_{t\leq T}(1-x)^t t^2\leq \frac{2}{x^3}$, we have
        \begin{align*}
            &\E{\norm{x^T-x^*}^2}\leq 2(1-\gamma\mu)^T\norm{x^0-x^*}^2\\ &+\frac{C\gamma^3\Bar{L}\tau_{mix}\sigma^2_*}{\mu^2\gamma^2}\\
        &+\frac{Cp_J^2\norm{P_{IS}-P_{\textit{L\'{e}vy}}}^2_{1}\sigma^2_*\Bar{L}^3}{\mu^3 L_{min}^2}. \end{align*}
        Choosing $\gamma=\min\{\frac{1}{\Bar{L}},\frac{1}{T\mu}\ln{T\frac{\norm{x^0-x^*}^2\mu^2}{C\tau_{mix}\sigma^2_*\Bar{L}}}\}$, we have
        the desired result.
\end{proof}
\subsection{Numerical Result}
 
 
\begin{figure*}[htbp]
\centering
\subfloat[]{\includegraphics[width=3in]{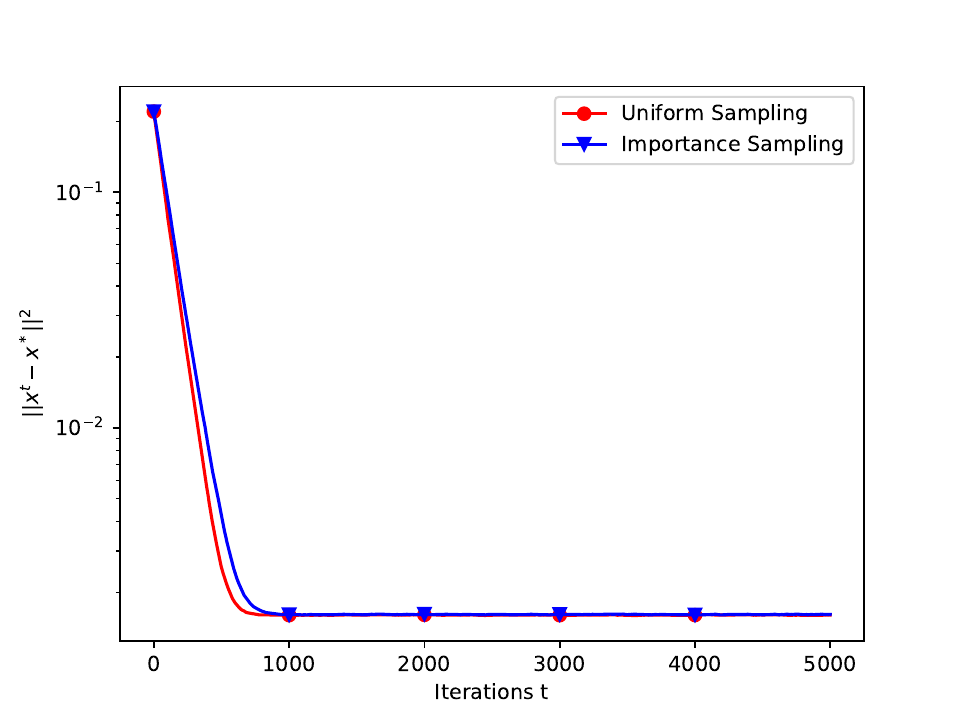}%
}
\hfil
\subfloat[]{\includegraphics[width=3in]{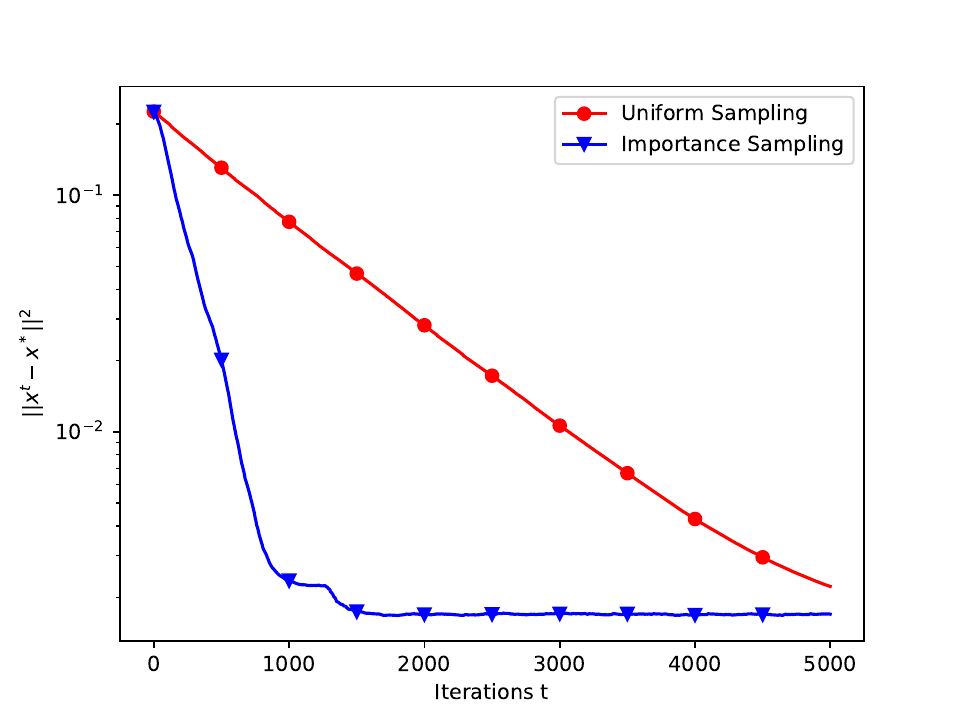}%
}
\caption{Regression model trained on a synthetic data set over a Erd\H{o}s-R\'{e}nyi (1000, 0.1) network with 1000
nodes. We compare the uniform sampling with Metropolis-Hastings transition probability and importance
sampling with Metropolis-Hastings transition probability. $\sigma^2_H=100$, $\sigma^2_L=1$. (a) Homogeneous Data. (b) Heterogeneous Data.}
\label{ER}
\end{figure*}

\begin{figure*}[htbp]
\centering
\subfloat[]{\includegraphics[width=3in]{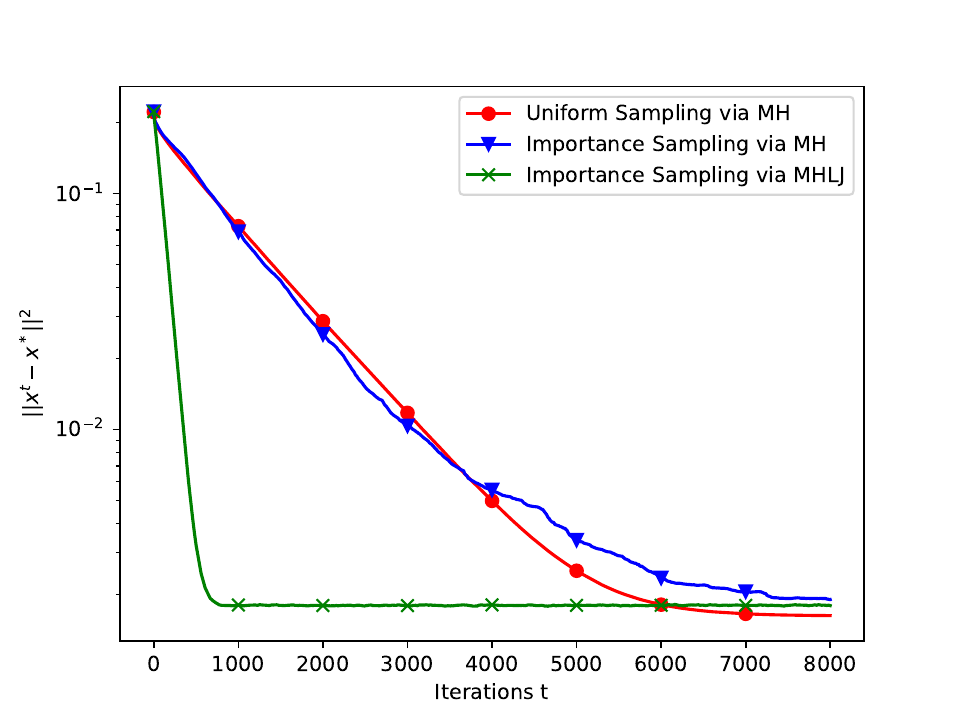}%
 }
\hfil
\subfloat[]{\includegraphics[width=3in]{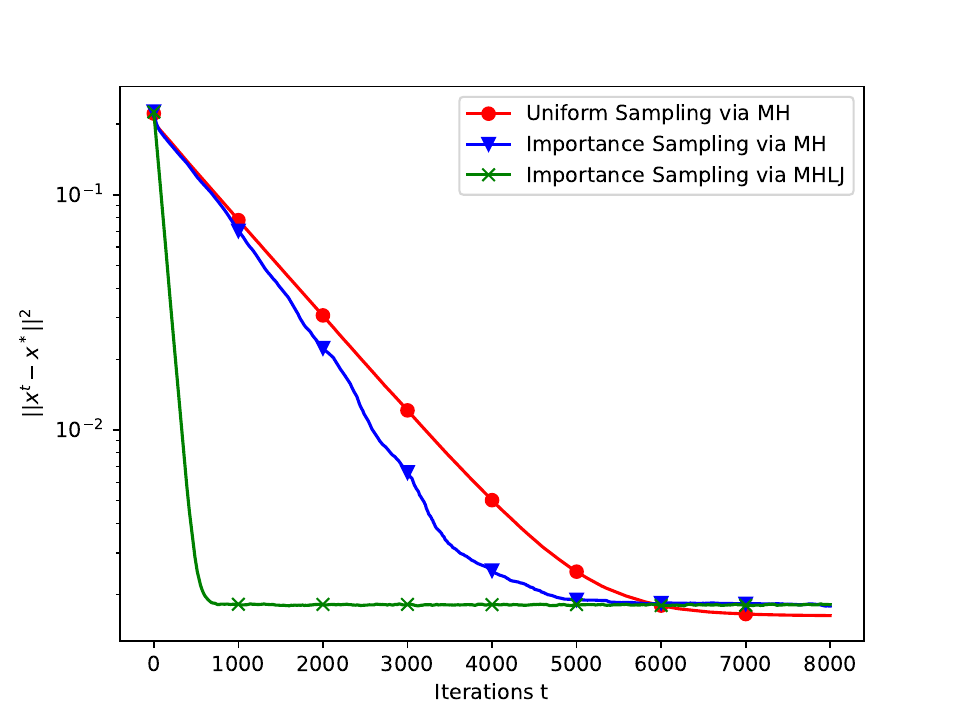}%
 }
\caption{Regression model
trained on a synthetic heterogeneous data set over sparse networks with 1000 nodes. We compare the
uniform sampling with Metropolis-Hastings transition probability, importance sampling with Metropolis-
Hastings transition probability and importance sampling with MHLJ. $\sigma^2_H=100$, $\sigma^2_L=1$. (a) 2-d grid. (b) Watts-Strogatz (1000, 4, 0.1) graph.}
\label{ERjump}
\end{figure*}
\begin{figure}[t]
    \centering
    \includegraphics[width=0.5\textwidth]{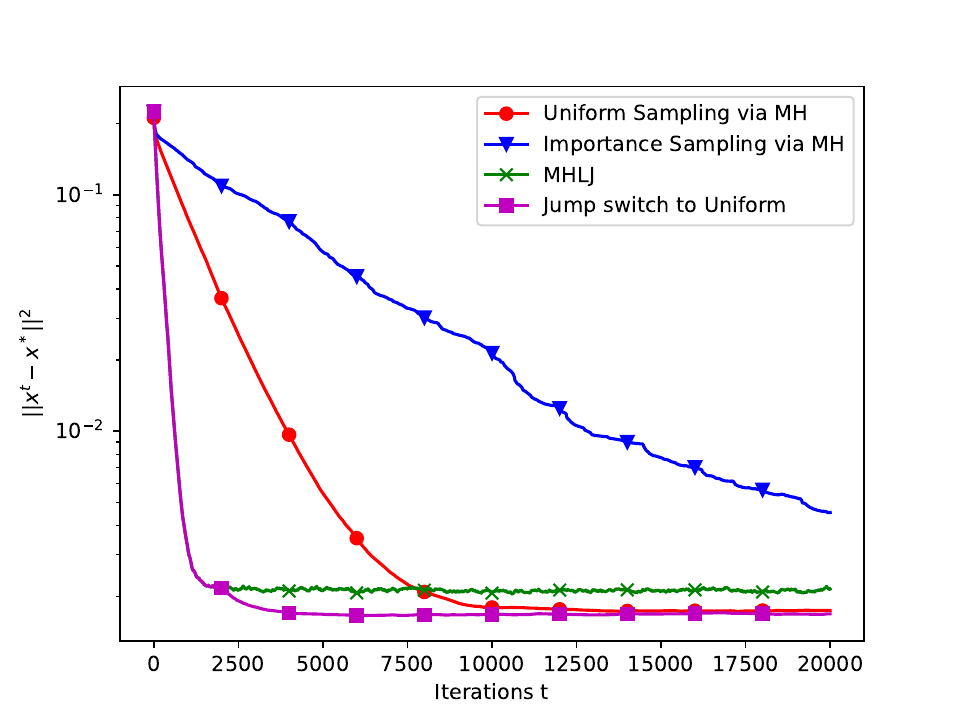}
    \caption{Shringking the jump probability $p_J$ towards zero eliminates the error gap introduced by making jumps without losing the convergence speed.}
    \label{fig:switch}
\end{figure}

We focus on the effect of designing the transition matrix on the convergence rate of decentralized SGD via random walk has the form as in \eqref{update}. We only consider time-homogeneous random walk learning in this work, that is the graph topology and transition probability of the random walk don't change over time. We compare our newly proposed algorithm to the Random Walk SGD algorithm using uniform sampling via \mh\  as in \cite{johansson2010randomized}; and using Importance sampling via \mh \ as in \cite{ayache2021private}, where the stationary distribution of the Metropolis-Hastings transition probability is the importance distribution $\pi_I(v)=\frac{L_v}{\sum_{v\in V}}$. 

We show the simulation results of the decentralized least square estimation problem on a graph $G=(V,E)$. The averaged loss function has the following formula:
\begin{align}
    f(x;\{A_v,y_v\}_{v\in V})=\frac{1}{|V|}\sum_{v\in V}(y_v-x^T A_v)^2,
\end{align}
where $\{A_v,y_v\}\in \mathbb{R}^d\times \mathbb{R}$ is the data stored at node $v$, the local loss function at node $v$ is:
\begin{align}
    f_v(x)=(y_v-x^TA_v)^2.
\end{align}
The local gradient Lipschitz constant is thus $L_v=2A_v^T A_v$.


We start from the Erd\H{o}s-R\'{e}nyi random graph. The simulation results show that random walk SGD using uniform sampling and importance sampling for the homogeneous data set have similar convergence rates, see  \cref{ER}.a. For the heterogeneous data set, importance sampling beats uniform sampling; see  \cref{ER}.b. However, when the data set is heterogeneous, and the graph is sparse, using importance sampling will lead to the entrapment problem, see  \cref{mainsim} in the main paper. We show first, on the ring graph, Importance sampling via \mh\  converges slower than Uniform sampling via \mh\ for the heterogeneous data set, which is contrary to what Importance sampling via \mh\  is designed for due to the entrapment. Our proposed algorithm converges significantly faster than Uniform sampling via \mh. Later, we show that the error gap can be eliminated by shrinking the probability of making jumps, $p_J$, towards zero; one note, in this case, is that by doing so, our algorithm converges to the correct optima without losing the convergence speed compared with the simple Metropolis-Hastings algorithm, see \cref{fig:switch}. We also consider other types of sparse random networks: 1. The 2-d grid graph, see \cref{ERjump}.a. 2. The Watts-Strogatz network, which is similar to regular graphs, has an average degree of order $\mathcal{O}(1)$, see \cref{ERjump}.b. We observe that for the heterogeneous data we consider, the entrapment problem happens on these sparse graphs and our MHLJ algorithm can overcome the entrapment problem and speed up the convergence.

We now give the detailed simulation setting.

\textit{Data:} 
The homogeneous data set $\{A_v,y_v\}_{v\in V}$ are generated in the following way:
\begin{itemize}
    \item $A_v\overset{\mathrm{i.i.d.}}{\sim} N_{10}(0,\sigma^2\mathbb{I}_{10})$.
    \item $y_v=A_v^T x+\epsilon$, where $\epsilon\overset{\mathrm{i.i.d.}}{\sim}N(0,1)$.
\end{itemize}
The heterogeneous data set $\{A_v,y_v\}_{v\in V}$ are generated in the following way: 
\begin{itemize}
    \item $A_v|\sigma^2\overset{\mathrm{i.i.d.}}{\sim} N_{10}(0,\sigma^2\mathbb{I}_{10})$, where $\sigma^2$ takes value $\sigma^2_L$ with probability $p=0.995$ and $\sigma^2_H$ with probability $p=0.005$.
    \item $y_v=A_v^T x+\epsilon$, where $\epsilon\overset{\mathrm{i.i.d.}}{\sim}N(0,1)$.
\end{itemize}
For each node $v$, we assign one data point $(X_v, y_v)$.

\textit{Step size:} We consider the constant step size to show the effect of importance sampling clearly. First, we choose the largest step size such that the random walk learning algorithm under uniform sampling converges. Then, we choose the step size such that the random walk learning algorithm under importance sampling converges to the same accuracy. Our new algorithm uses the same step size as in the importance sampling. 


 
\end{document}